\documentclass{article} % For LaTeX2e
\usepackage{nips10submit_e,times}

\usepackage{bbm}
\usepackage{amsfonts}
\usepackage{amsmath}
\usepackage{amssymb}
\usepackage{longtable}
\usepackage{tabularx}
\usepackage{multirow}
\usepackage{graphicx}
\usepackage{algorithm}
\usepackage{algorithmic}
\usepackage{url}

\newcommand{\Ctpos}[1][f]{C_{#1}^t}
\newcommand{\onespect}{_{1}}

\newenvironment{proof2}[1]{\par\noindent{\bf Proof #1:\ }}{\hfill$\Box$\medskip}

\newtheorem{lemma}{Lemma}[section]
\newtheorem{proposition}{Proposition}[section]
\newtheorem{theorem}{Theorem}[section]

\newtheorem{definition}{Definition}[section]

\newcommand{\abs}[1]{\ensuremath {\left| #1 \right|}}
\newcommand{\norm}[1]{\ensuremath{\left\| #1 \right\|}}

\newcommand{\curlybrackets}[1]{\ensuremath{\left\{ #1 \right\}}}

\newcommand{\inner}[1]{\left\langle#1\right\rangle}

\def\R{\mathbb{R}}

\def\N{\mathbb{N}}

\def\median{\mathop{\rm median}\nolimits}

\def\argmax{\mathop{\rm arg\,max}\limits}%    a math operator.
\def\argmin{\mathop{\rm arg\,min}\limits}%    a math operator.
\def\sign{\mathop{\rm sign}\limits}

\def\cut{\mathrm{cut}}

\def\Rcut{\mathrm{RCut}}
\def\RCC{\mathrm{RCC}}

\def\min{\mathop{\rm min}\nolimits}
\def\max{\mathop{\rm max}\nolimits}
\def\ones{\mathbf{1}}

\newenvironment{proof}{\par\noindent{\bf Proof:\ }}{\hfill$\Box$\\[2mm]}

\title{An Inverse Power Method for Nonlinear Eigenproblems with Applications in \\$1$-Spectral Clustering and Sparse PCA}

\author{
Matthias Hein \hspace{0.5cm}  Thomas B\"uhler\\
Saarland University, Saarbr\"ucken, Germany\\
\texttt{\{hein,tb\}@cs.uni-saarland.de} 
}

\nipsfinalcopy % Uncomment for camera-ready version

\begin{document}
\title{An Inverse Power Method for Nonlinear Eigenproblems with Applications in \\$1$-Spectral Clustering and Sparse PCA}

\maketitle

\begin{abstract}
Many problems in machine learning and statistics can be formulated as (generalized) eigenproblems. In terms of the associated
optimization problem, computing linear eigenvectors amounts to finding critical points of a quadratic function subject to quadratic constraints. 
In this paper we show that a certain class of constrained optimization problems with nonquadratic objective and constraints
can be understood as \emph{nonlinear eigenproblems}. We derive a generalization of the inverse power method
which is guaranteed to converge to a nonlinear eigenvector.
We apply the inverse power method to 1-spectral clustering and sparse PCA which can naturally be formulated as nonlinear eigenproblems.
In both applications we achieve state-of-the-art results in terms of solution quality and runtime. 
Moving beyond the standard eigenproblem should be useful also in many other applications and our inverse power method
can be easily adapted to new problems. 
\end{abstract}

\section{Introduction}
Eigenvalue problems associated to a symmetric and positive semi-definite matrix are quite abundant in machine learning and statistics. 
However, considering the eigenproblem from a variational point of view using Courant-Fischer-theory, the objective is a ratio of
quadratic functions, which is quite restrictive from a modeling perspective. We show in this paper that using a ratio of $p$-homogeneous functions leads quite naturally to a
\emph{nonlinear} eigenvalue problem, associated to a certain nonlinear operator. Clearly, such a generalization is only interesting if certain properties
of the standard problem are preserved and efficient algorithms for the computation of nonlinear eigenvectors are available. In this paper we present an
efficient generalization of the inverse power method (IPM) to nonlinear eigenvalue problems and study the relation to the standard problem.
While our IPM is a general purpose method, we show for two unsupervised learning problems that it can be easily adapted to a particular application.

The first application is spectral clustering \cite{Lux07}. In prior work \cite{BH09} we proposed $p$-spectral clustering based on the graph $p$-Laplacian, a 
nonlinear operator on graphs which reduces to the standard graph Laplacian for $p=2$. For $p$ close to one, we obtained much better cuts than standard spectral
clustering, at the cost of higher runtime.
Using the new IPM, we efficiently compute eigenvectors of the $1$-Laplacian for $1$-spectral clustering. Similar to the recent work of \cite{SB10},  %\cite{BH09} 
we improve considerably compared to \cite{BH09} both in terms of runtime and the achieved Cheeger cuts. However, opposed to the suggested method in \cite{SB10} our IPM is guaranteed
to converge to an eigenvector of the $1$-Laplacian.

The second application is sparse Principal Component Analysis (PCA). The motivation for sparse PCA is that the largest PCA component is difficult to interpret as usually all components are nonzero.
In order to allow a direct interpretation it is therefore desirable to have only a few features with nonzero components but which still explain most of the variance. This kind of trade-off has
been widely studied in recent years, see \cite{Jou10} and references therein. We show that also sparse PCA has a natural formulation as a nonlinear eigenvalue problem and can be efficiently
solved with the IPM.

\section{Nonlinear Eigenproblems}\label{sec:nonlinear_eigenproblems}

The standard eigenproblem for a symmetric matric $A \in \R^{n\times n}$ is of the form
\begin{equation}\label{eq:linear_eigenproblem}
	A f - \lambda f = 0 , 
\end{equation}
where $f\in \R^n$ and $\lambda \in \R$. It is a well-known result from linear algebra that for symmetric matrices $A$, the eigenvectors of $A$  
can be characterized as critical points of the functional 
\begin{equation}\label{eq:StdFct}
	F_{\mathrm{Standard}}(f)=\frac{\inner{f,Af}}{\norm{f}^2_2}.
\end{equation}
The eigenvectors of $A$ can be computed using the Courant-Fischer Min-Max principle. While the ratio of quadratic functions is useful in several applications, it is a severe modeling restriction. This restriction however can be overcome using nonlinear eigenproblems.
In this paper we consider functionals $F$ of the form
\begin{equation}\label{eq:GenFct}
	F(f) = \frac{R(f)}{S(f)},
\end{equation}
where with $\R_+=\{x \in \R \,|\, x \geq 0\}$ we assume $R:\R^n \rightarrow \R_+$, $S:\R^n \rightarrow \R_+$ to be convex, Lipschitz continuous, even and positively $p$-homogeneous\footnote{A function $G:\R^n \rightarrow \R$ is positively homogeneous of degree $p$ if
$G(\gamma x)=\gamma^p G(x)$ for all $\gamma \geq 0$.} with $p\geq 1$. Moreover, we assume that $S(f)=0$ if and only if $f=0$. 
The condition that $R$ and $S$ are $p$-homogeneous and even will imply for any eigenvector $v$ that also $\alpha v$ for $\alpha \in \R$ is an eigenvector. 
It is easy to see that the functional of the standard eigenvalue problem in Equation \eqref{eq:StdFct} is a special case of the general functional in \eqref{eq:GenFct}.
	
To gain some intuition, let us first consider the case where $R$ and $S$ are differentiable. Then it holds for every critical point $f^*$ of $F$,
\[
	 \nabla F(f^*)= 0 \qquad \Longleftrightarrow \qquad \nabla R(f^*) - \frac{R(f^*)}{S(f^*)} \cdot \nabla S(f^*) = 0 \ .
\]
Let $r,s:\R^n \rightarrow \R^n$ be the operators defined as $r(f) = \nabla R(f)$, $s(f) = \nabla S(f)$ and $\lambda^*= \frac{R(f^*)}{S(f^*)}$, we see that every critical point $f^*$
of $F$ satisfies the \emph{nonlinear eigenproblem} 
\begin{equation}\label{eq:NonlinearEG}
	r(f^*) - \lambda^*\, s(f^*) = 0,
\end{equation}
which is in general a system of nonlinear equations, as $r$ and $s$ are nonlinear operators. If $R$ and $S$ are both quadratic, $r$ and $s$ are linear operators
and one gets back the standard eigenproblem \eqref{eq:linear_eigenproblem}.

Before we proceed to the general nondifferentiable case, we have to introduce some important concepts from nonsmooth analysis. Note that $F$ is in general nonconvex and nondifferentiable. In the following we denote by $\partial F(f)$ the \emph{generalized gradient} of $F$ at $f$ according to Clarke \cite{Cla83},
\[ \partial F(f) = \{ \xi \in \R^n \; \big| \; F^0(f,v) \, \geq \, \inner{\xi,v}, \quad \textrm{ for all } v \in \R^n\},\]
where $F^0(f,v)= \lim_{g \rightarrow f, \,t\rightarrow 0} \sup \frac{F(g+tv) - F(g)}{t}$. In the case where $F$ is convex, $\partial F$ is the subdifferential of $F$ and $F^0(f,v)$ the 
directional derivative for each $v \in \R^n$. A characterization of critical points of nonsmooth functionals is as follows.
%a generalization of the well-known subdifferential to the convex case wich yields the standard subdifferential as a special case if $F$ is convex, and the gradient if $F$ is differentiable. 
\begin{definition}[\cite{Cha81}]
A point $f \in \R^n$ is a critical point of $F$, if $0 \in \partial F$.
\end{definition} 
This generalizes the well-known fact that the gradient of a differentiable function vanishes at a critical point.
We now show that the nonlinear eigenproblem \eqref{eq:NonlinearEG} is a necessary condition for a critical point and in some
cases even sufficient. A useful tool is the generalized Euler identity.
\begin{theorem}[\cite{YW08}]\label{th:EulerId}
Let $R:\R^n \rightarrow \R$ be a positively $p$-homogeneous and convex continuous function. 
Then, for each $x \in \R^n$ and $r^* \in \partial R(x)$ it holds that
$\inner{x,r^*}=p\,R(x)$.
%Then, for each $x \in \R^n$
%the following identity holds:
%\[ \inner{x,r^*}=p\,R(x), \quad \textrm{ for all } r^* \in \partial R(x).\]
\end{theorem}
The next theorem characterizes the relation between nonlinear eigenvectors and critical points of $F$.
\begin{theorem}\label{th:CriticalPoint}
Suppose that $R,S$ fulfill the stated conditions. Then a necessary condition for $f^*$ being a critical point of $F$ is
\begin{equation}\label{eq:GenNonlinearEG} 
 0 \in \partial R(f^*) - \lambda^*\, \partial S(f^*) , \qquad \text{where } \qquad \lambda^* = \frac{R(f^*)}{S(f^*)}.
\end{equation}
If $S$ is continuously differentiable at $f^*$, then this is also sufficient. %In addition, $\lambda^* = \frac{R(f^*)}{S(f^*)}$. 
\end{theorem}
%%%
%%% START Comment for short paper 
%%%
\begin{proof}
Let $f^*$ fulfill the general nonlinear eigenproblem in \eqref{eq:GenNonlinearEG}, where $r^* \in \partial R(f^*), s^* \in \partial S(f^*)$, such that $r^*-\lambda^*\, s^*=0$. Then by Theorem \ref{th:EulerId},
\[ 0=\inner{f^*,r^*} - \lambda^* \inner{f^*,s^*} = p\,R(f^*) - p\,\lambda^*\, S(f^*),\]
and thus $\lambda^* = R(f^*)/S(f^*)$. As $R,S$ are Lipschitz continuous, we have, see Prop. 2.3.14 in \cite{Cla83},
\begin{equation}\label{eq:SubGradRatio} \partial \Big(\frac{R}{S}\Big)(f) \; \subseteq \; \frac{S(f)\, \partial R(f) - R(f)\,\partial S(f)}{S(f)^2}.\end{equation}
Thus if $f^*$ is a critical point, that is $0 \in \partial F(f^*)$, then $0 \in \partial R(f^*) - \frac{R(f^*)}{S(f^*)}\,\partial S(f^*)$
given that $f^* \neq 0$. Moreover, by Prop. 2.3.14 in \cite{Cla83} we have equality in \eqref{eq:SubGradRatio}, if $S$ is continuously differentiable at $f^*$ and
thus \eqref{eq:GenNonlinearEG} implies that $f^*$ is a critical point of $F$.
\end{proof}
%%%%
%%% END Comment for short paper 
%%%

Finally, the definition of the associated nonlinear operators in the nonsmooth case is a bit tricky as $r$ and $s$ can be set-valued. However, as we assume $R$ and $S$ to
be Lipschitz, the set where $R$ and $S$ are nondifferentiable has measure zero and thus $r$ and $s$ are single-valued almost everywhere.

\section{The inverse power method for nonlinear Eigenproblems}\label{sec:algorithm}

A standard technique to obtain the smallest eigenvalue of a positive semi-definite symmetric matrix $A$ is the inverse power method \cite{Gol96}. 
Its main building block is the fact that the iterative scheme
\begin{equation}\label{eq:invpow_linear}
	 A f^{k+1} = f^k 
\end{equation}
converges to the smallest eigenvector of $A$. Transforming \eqref{eq:invpow_linear} into the optimization problem
\begin{equation}\label{eq:invpow_linear_opt}
 f^{k+1}=\argmin_{u} \frac{1}{2} \inner{u,A \,u} - \inner{u,f^k}
\end{equation}
is the motivation for the general IPM. The direct generalization tries to solve
\begin{equation}\label{eq:NonlinearEG2}
	0 \in r(f^{k+1}) - s(f^k) \quad \textrm{ or equivalently }\quad f^{k+1}=\argmin_{u} R(u) - \inner{u,s(f^k)},
\end{equation}
where $r(f) \in \partial R(f)$ and $s(f) \in \partial S(f)$. For $p>1$ this leads directly to Algorithm \ref{alg:general}, however for $p=1$ the direct generalization fails.
In particular, the ball constraint has to be introduced in Algorithm \ref{alg:P1} as the objective in the optimization problem \eqref{eq:NonlinearEG2} is otherwise unbounded from below. (Note that the 2-norm is only chosen for algorithmic convenience). Moreover, the introduction of $\lambda_k$ in Algorithm \ref{alg:P1} is necessary to guarantee descent whereas in Algorithm \ref{alg:general} it would just yield a rescaled solution of the problem in the inner loop (called inner problem in the following).

For both methods we show convergence to a solution of \eqref{eq:NonlinearEG}, which  by Theorem \ref{th:CriticalPoint} is a necessary condition for a critical point of $F$ and often also sufficient. Interestingly, both applications are naturally formulated as $1$-homogeneous problems so that we use in both cases Algorithm \ref{alg:P1}. Nevertheless, we state the second algorithm for completeness. Note that we cannot guarantee convergence to the smallest eigenvector even though our experiments suggest that we often do so. However, as the method is fast one can afford to run it multiple times with different initializations and use the eigenvector with smallest eigenvalue. 

\begin{algorithm}[htb]
   \caption{Computing a nonlinear eigenvector for convex positively $p$-homogeneous functions $R$ and $S$ with $p=1$}
   \label{alg:P1}
\begin{algorithmic}[1]
   %\STATE {\bfseries Input:}  accuracy $\epsilon$
   \STATE {\bfseries Initialization:} $f^0 = \text{random}$ with $\norm{f^0} = 1$, $\lambda^0 =F(f^0)$
   \REPEAT
   \STATE $f^{k+1} = \argmin_{\norm{u}_2 \leq 1} \left\{ R(u) - \lambda^k \inner{u, s(f^k)}    \right\}$ 
   			\hspace{1cm} \text{where\ } $s(f^k) \in \partial S(f^k)$
   %\STATE $f^{k+1} = g^{k+1}/S(g^{k+1})^{1/p}$
  % \STATE $f^{k+1} = g^{k+1}/\norm{g^{k+1}}_2$
   \STATE $\lambda^{k+1}= R(f^{k+1})/S(f^{k+1})$
	\UNTIL $\frac{\abs{\lambda^{k+1}-\lambda^k}}{\lambda^k}< \epsilon$
	\STATE {\bfseries Output:} eigenvalue $\lambda^{k+1}$ and eigenvector $f^{k+1}$.
\end{algorithmic}
\end{algorithm}

\begin{algorithm}[htb]
   \caption{Computing a nonlinear eigenvector for convex positively $p$-homogeneous functions $R$ and $S$ with $p>1$}
   \label{alg:general}
\begin{algorithmic}[1]
   %\STATE {\bfseries Input:} accuracy $\epsilon$
   \STATE {\bfseries Initialization:} $f^0 = \text{random}$, $\lambda^0 =F(f^0)$
   \REPEAT
   \STATE $g^{k+1} = \argmin_{u} \left\{ R(u) - \inner{u, s(f^k)}    \right\}$ 
   			\hspace{1cm} \text{where\ } $s(f^k) \in \partial S(f^k)$
   \STATE $f^{k+1} = g^{k+1}/S(g^{k+1})^{1/p}$
   \STATE $\lambda^{k+1}= R(f^{k+1})/S(f^{k+1})$
	\UNTIL $\frac{\abs{\lambda^{k+1}-\lambda^k}}{\lambda^k}< \epsilon$
	\STATE {\bfseries Output:} eigenvalue $\lambda^{k+1}$ and eigenvector $f^{k+1}$.
\end{algorithmic}
\end{algorithm}

The inner optimization problem is convex for both algorithms.
%Note that the inner objective is convex, thus in principle it can be solved by any standard technique  for convex programming \cite{Ber99}, or by standard tools such as CVX \cite{Gra10}. 
In turns out that both for $1$-spectral clustering and sparse PCA the inner problem can be solved very efficiently, for sparse PCA it has even a closed form solution. 
While we do not yet have results about convergence speed,  empirical observation shows that one usually converges quite quickly to an eigenvector.

To our best knowledge both suggested methods have not been considered before. In \cite{BieErcMar09} they propose an inverse power method specially tailored towards
the continuous $p$-Laplacian for $p>1$, which can be seen as a special case of Algorithm \ref{alg:general}. In \cite{Jou10} a generalized power method has been proposed
which will be discussed in Section \ref{sec:app_sparse_pca}.
Finally, both methods can be easily adapted to compute the largest nonlinear eigenvalue, which however we have to omit due to space constraints.

\begin{lemma}\label{le:monotony_funct}
The sequences $f^k$ produced by Alg.~\ref{alg:P1} and \ref{alg:general} satisfy $F(f^k) > F(f^{k+1})$ for all $k\geq 0$ or the sequences terminate.
\end{lemma}

\begin{theorem}\label{th:convergence_eigenvector}
The sequences $f^k$ produced by Algorithms~\ref{alg:P1} and \ref{alg:general} converge to an eigenvector $f^*$ with eigenvalue $\lambda^* \in \left[ 0, F(f^0)\right]$ in the sense
that it solves the nonlinear eigenproblem \eqref{eq:GenNonlinearEG}. If $S$ is continuously differentiable at $f^*$, then $F$ has a critical point at $f^*$.
\end{theorem}

%%%
%%% START Comment for short paper 
%%%
Throughout the proofs, we use the notation $\Phi_{f^k}(u) = R(u) - \lambda^k \inner{u, s(f^k)}$ and \ $\Psi_{f^k}(u) = R(u) - \inner{u, s(f^k)}$ for the objectives of the inner problems in Algorithms \ref{alg:P1} \& \ref{alg:general}, respectively.

\begin{proof2}{of Lemma~\ref{le:monotony_funct} for Algorithm~\ref{alg:P1}}
First note that the optimal value of the inner problem is non-positive as $\Phi_{f^k}(0)=0$. Moreover, as $\Phi_{f^k}$ is $1$-homogeneous, the 
minimum of $\Phi_{f^k}$ is always attained at the boundary of the constraint set.
Thus any $f^k$ fulfills $\norm{f^k}_2^2=1$ and thus is feasible, and 
\[
	\min_{\norm{f}_2^2 \leq 1} {\Phi_{f^k}(f)} \leq \Phi_{f^k}(f^k) = R(f^k) - \lambda^k\inner{f^k, s(f^k)} = R(f^k) - F(f^k)\cdot S(f^k) = 0 \ ,
\]
where we used $\inner{f^k, s(f^k)} = S(f^k)$ %for $s(f^k) \in \partial S(f^k)$ 
from Theorem \ref{th:EulerId}. If the optimal value is zero, then $f^k$ is a possible minimizer and the sequence terminates and $f^k$
is an eigenvector see proof of Theorem~\ref{th:convergence_eigenvector} for Algorithm~\ref{alg:P1}. Otherwise the optimal value is negative and
at the optimal point $f^{k+1}$ we get $R(f^{k+1}) < \lambda^k \inner{f^{k+1}, s(f^k)}$.
The definition of the subdifferential $s(f^k)$ together with the $1$-homogeneity of $S$ yields 
\[
	S(f^{k+1}) \geq S(f^k) + \inner{f^{k+1}-f^k, s(f^k)} = \inner{f^{k+1}, s(f^k)} \ ,
	\] 
and finally $F(f^{k+1})=\frac{R(f^{k+1})}{S(f^{k+1})} < \lambda^k = F(f^k)$.
\end{proof2}
\begin{proof2}{of Theorem~\ref{th:convergence_eigenvector} for Algorithm~\ref{alg:P1}}
By Lemma~\ref{le:monotony_funct} the sequence $F(f^k)$ is monotonically decreasing. By assumption $S$ and $R$ are nonnegative and hence $F$ is bounded below by zero. 
Thus we have convergence towards a limit
\[
	\lambda^* = \lim_{k \rightarrow \infty} F(f^k) \ .
\]
Note that $\norm{f^{k}}_2^2 \leq 1$ for every $k$, thus the sequence $f^k$ is contained in a compact set, which implies that there exists a subsequence $f^{k_j}$ converging to some element $f^*$. As the sequence $F(f^{k_j})$ is a subsequence of a convergent sequence, it has to converge towards the same limit, hence also
\[
	\lim_{j \rightarrow \infty} F(f^{k_j}) = \lambda^* \ .
\]
As shown before, the objective of the inner optimization problem is nonpositive at the optimal point. Assume now that $\min_{\norm{f}_2^2 \leq 1} \Phi_{f^*}(f)<0$. Then the vector $f^{**}  = \argmin_{\norm{f}_2^2 \leq 1} \Phi_{f^*}(f)$ satisfies
\[
	R(f^{**}) < \lambda^* \inner{f^{**},s(f^*)} = \lambda^*\left(S(f^*) + \inner{f^{**}-f^*,s(f^*)}\right) \leq \lambda^* S(f^{**}) \ ,
\]
where we used the definition of the subdifferential and the $1$-homogeneity of $S$. Hence
	\[
		F(f^{**}) < \lambda^* = F(f^*)\ ,
		\]
which is a contradiction to the fact that the sequence $F(f^k)$ has converged to $\lambda^*$. Thus we must have $\min_{\norm{f}_2^2 \leq 1} \Phi_{f^*}(f)=0$,	i.e. the function $\Phi_{f^*}(f)$ is nonnegative in the unit ball. Using the fact that for any $\alpha \geq 0$, 
\[
	\Phi_{f^*}(\alpha f) = \alpha \Phi_{f^*} \ ,
\]
we can even conclude that the function $\Phi_{f^*}(f)$ is nonnegative everywhere, and thus $\min_{f} \Phi_{f^*}(f)=0$. Note that $\Phi_{f^*}(f^*)=0$, which implies that $f^*$ is a global minimizer of $\Phi_{f^*}$, and hence
\[
	0 \in \partial \Phi_{f^*}(f^*) = \partial R(f^*) - \lambda^* \partial S(f^*) \ ,
\]
which implies that $f^*$ is an eigenvector with eigenvalue $\lambda^*$. Note that this argument was independent of the choice of the subsequence, thus every convergent subsequence converges to an eigenvector with the same eigenvalue $\lambda^*$. Clearly we have $\lambda^* \leq F(f^0)$.
\end{proof2}

The following lemma is useful in the convergence proof of Algorithm~\ref{alg:general}.
\begin{lemma}\label{lemma:homogeneity_subgradient}
Let $R$ be a convex, positively $p$-homogeneous function with $p\geq 1$. Then for any $x \in \R^n$, $t\geq 0$ and any $r^* \in \partial R(x)$ we have $t^{p-1} r^* \in  \partial R(tx)$.
\end{lemma}
\begin{proof}
Using the definition of the subgradient, we have for any $y \in \R^n$ and any $t\geq 0$,
\[ t^p R(y) \geq t^p R(x) + t^p \inner{r^*, y-x}  \ . \]
Using the $p$-homogeneity of $R$, we can rewrite this as
\[
	R(ty) \geq R(tx) +  \inner{t^{p-1}r^*, ty-tx} \ ,
\]
which implies $t^{p-1} r^* \in \partial R(tx)$.
\end{proof}

The following Proposition generalizes a result by Zarantonello \cite{Zar76}. 
\begin{proposition}\label{lemma:HoelderIneq}
Let $R:\R^n \rightarrow \R$ be a convex, continuous and positively $p$-homogeneous and even functional and $\partial R(f)$ its subdifferential at $f$. Then it holds for any $f,g \in R^n$ and $r(f)\in \partial R(f)$,
\[
	\abs{\inner{r(f),g}} \leq \inner{r(f),f}^{1-\frac{1}{p}} \inner{r(g),g}^{\frac{1}{p}} = p \cdot R(f)^{1-\frac{1}{p}} R(g)^{\frac{1}{p}} \ .
\]
\end{proposition}
\begin{proof}
First observe that for any  $k$ points $x_0, \dots x_{k-1} \in \R^n$, the subdifferential inequality yields
\begin{align*}
	R(x_l) &\geq R(x_{l-1}) + \inner{r(x_{l-1}),x_l-x_{l-1}} \ , \; \forall 1 \leq l \leq k-1\\
	R(x_0) &\geq R(x_{k-1}) + \inner{r(x_{k-1}),x_0-x_{k-1}} \ ,
\end{align*}
and hence, by summing up,
\begin{equation}\label{eq:cyclically_monotone}
		\inner{r(x_{0}),x_1-x_{0}} + \dots + \inner{r(x_{k-2}),x_{k-1}-x_{k-2}} + \inner{r(x_{k-1}),x_0-x_{k-1} } \leq 0 \ .
\end{equation}
Let now $f,g \in \R^n$, and $r(f) \in \partial R(f), r(g) \in \partial R(g)$. We construct a set of $2m$ points $x_0, \dots x_{2m-1}$ in $\R^n$, where $m\in \N$, as follows:
\begin{align*}
	x_i = \left\{ \begin{array}{ll}
							\frac{i+1}{m}  f & , \ 0 \leq i \leq m-1 \\
							\frac{2m-1-i}{m} g & , \ m \leq i \leq 2m-1 
	 \end{array}
	 \right. \ .
\end{align*}
By Lemma \ref{lemma:homogeneity_subgradient} for all $i\in \{0, \dots 2m-1 \}$ there exists an $r^*(x_i) \in \partial R(x_i)$ s.t.
\begin{align*}
	r^*(x_i) = \left\{ \begin{array}{ll}
							\left(\frac{i+1}{m}\right)^{p-1} r(f) & , \ 0 \leq i \leq m-1 \\
							\left(\frac{2m-1-i}{m}\right)^{p-1} r(g) & , \ m \leq i \leq 2m-1 
	 \end{array}
	 \right. \ .
\end{align*}
Eq. \eqref{eq:cyclically_monotone} now yields
\begin{align*}
\sum_{i=0}^{m-2}   \inner{\left(\frac{i+1}{m}\right)^{p-1}  r(f),\frac{1}{m} f} +  
  \inner{r(f),\frac{m-1}{m} g-f}  & \\
	-  \sum_{i=m}^{2m-2}  \inner{\left(\frac{2m-1-i}{m}\right)^{p-1} r(g),\frac{1}{m}g} 
	+    \inner{0 \cdot  r(g),\frac{1}{m} f- 0 \cdot g} &  \leq 0
\end{align*}
which simplifies to
\begin{align}\label{eq:hoelder_proof_ineq}
	\left(\frac{1}{m} \sum_{j=1}^{m-1} \left(\frac{j}{m}\right)^{p-1} - 1 \right) \inner{r(f),f}    
	- \left(\frac{1}{m} \sum_{j=1}^{m-1} \left(\frac{j}{m}\right)^{p-1}  \right) \inner{r(g),g} 
	 + \frac{m-1}{m} 	 \inner{r(f),g}   \leq 0 \ .  
	\end{align}
By letting $m\rightarrow \infty$ we obtain for the two sums
\[
	\lim_{m \rightarrow \infty} \left(\frac{1}{m} \sum_{j=1}^{m-1} \left(\frac{j}{m}\right)^{p-1} \right)
	= \lim_{m \rightarrow \infty} \left(\frac{1}{m} \sum_{j={\frac{1}{m},\frac{2}{m}, \dots}}^{\frac{m-1}{m}} j^{p-1} \right)
	= \int_0^1 j^{p-1} dj = \frac{1}{p} \ . 
\]
Hence in total in the limit $m \rightarrow \infty$ Eq. \eqref{eq:hoelder_proof_ineq} becomes
\[
	\inner{r(f),g} - 	\left(1- \frac{1}{p} \right) \inner{r(f),f} - \frac{1}{p} \inner{r(g),g} \leq 0 \ .
\]
As the above inequality holds for all $f,g \in \R^n$, clearly we can now perform the substitution $f \rightarrow t^{-1}f, g \rightarrow t^{p-1}g, r(f) \rightarrow t^{-(p-1)} r(f), r(g) \rightarrow t^{(p-1)^2} r(g)$, where $t\in \R^+$, which gives
\begin{equation}\label{eq:hoelder_gen_substitution}
	\inner{r(f),g} - 	\left(1- \frac{1}{p} \right) t^{-p}\inner{r(f),f} - \frac{1}{p} t^{p(p-1)}\inner{r(g),g} \leq 0 \ .
\end{equation}
A local optimum with respect to $t$ of the left side satisfies the necessary condition 
\begin{align*}
	0  = & \ \left(p-1 \right)t^{-p-1}\inner{r(f),f}  - (p-1) t^{p^2-p-1}\inner{r(g),g}\\
	 = & \ t^{-p-1} (p-1) \left(\inner{r(f),f}  -  t^{p^2}\inner{r(g),g} \right) \ ,
\end{align*}
which implies that
\[
	t^{p} = \left(\frac{\inner{r(f),f}}{\inner{r(g),g}}\right)^{\frac{1}{p}} \ .
\]
Plugging this into \eqref{eq:hoelder_gen_substitution} yields
\begin{align*}
	0  \geq & \inner{r(f),g} - 	\left(1- \frac{1}{p} \right) \inner{r(g),g}^{\frac{1}{p}} \inner{r(f),f}^{1-\frac{1}{p}} 
	- \frac{1}{p}\inner{r(f),f}^{1- \frac{1}{p}} \inner{r(g),g}^{\frac{1}{p}} \\
	  = & 	 \inner{r(f),g} - \inner{r(f),f}^{1- \frac{1}{p}} \inner{r(g),g}^{\frac{1}{p}} \ .
\end{align*}
By the homogeneity of $R$ we then have
\[
	\inner{r(f),g} \leq \inner{r(f),f}^{1- \frac{1}{p}} \inner{r(g),g}^{\frac{1}{p}} = p \cdot R(f)^{1-\frac{1}{p}} R(g)^{\frac{1}{p}} \ .
\]
Finally, note that we can replace the left side by its absolute value since replacing $g$ with $-g$ yields
\[
	\inner{r(f),-g} \leq p \cdot R(f)^{1-\frac{1}{p}} R(-g)^{\frac{1}{p}}= p \cdot R(f)^{1-\frac{1}{p}} R(g)^{\frac{1}{p}} \ ,
\]
where we used the fact that $R$ is even. 
\end{proof}

\begin{proof2}{of Lemma \ref{le:monotony_funct} for Algorithm~\ref{alg:general}}
Note that as $R(u)\geq 0$, the minimum of the objective of the inner problem is attained for some $u$ with $\inner{u,s(f^k)}>0$. Choose $u$ such that $\inner{u,s(f^k)}>0$. Then we minimize $\Psi_{f^k}$ on the ray $t u$, $t \geq 0$. We have 
\[ \Psi_{f^k}(tu) = R(tu) - \inner{t\,u,s(f^k)} = t^p \,R(u) - t\inner{u,s(f^k)} \]
and hence 
\[\frac{\partial}{\partial t} \Psi_{f^k}(tu) = p \,t^{p-1} R(u) - \inner{u,s(f^k)}\]
and thus the minimum is attained at $t^*(u) = \Big(\frac{\inner{u,s(f^k)}}{p\,R(u)}\Big)^{\frac{1}{p-1}}>0$ and 
\[ \Psi_{f^k}(t^*(u) u) = t^*(u)^p R(u) - t^*(u) \inner{u,s(f^k)} = (1-p) \Big(\frac{\inner{u,s(f^k)}^p}{p^p\,R(u)}\Big)^\frac{1}{p-1}.\]
Assume there exists $u$ that satisfies $\Psi_{f^k}(u) < \Psi_{f^k}(\hat{f}) $ where $\hat{f}=F(f^k)^{\frac{1}{1-p}}\,f^k$. 
Hence, also $\Psi_{f^k}(t^*(u)u) < \Psi_{f^k}(\hat{f})$, which implies 
\begin{align*}
   (1-p) \Big(\frac{\inner{u,s(f^k)}^p}{p^p\,R(u)}\Big)^\frac{1}{p-1} & <
    F(f^k)^{\frac{p}{1-p}} R(f^k) - F(f^k)^{\frac{1}{1-p}}\inner{f^k,s(f^k)}\\
  &= F(f^k)^{\frac{1}{1-p}}(1-p)  \ ,
\end{align*}
where we used the fact that $\inner{f^k,s(f^k)} = p S(f^k)$ and $S(f^k)=1$. Rearranging, we obtain
\[ 	F(f^k) > \frac{p^p R(u)}{\inner{u,s(f^k)}^p} \ . \]
Using the H{\"o}lder-type inequality of Proposition \ref{lemma:HoelderIneq} and $S(f^k)=1$, we obtain
\[ \inner{u,s(f^k)} \, \leq p S(f^k)^{1-\frac{1}{p}}S(u)^{\frac{1}{p}} = p S(u)^{\frac{1}{p}},\]
which gives $F(f^k) > F(u)$.
Let now $f^*$ be the minimizer of $\Psi_{f^k}$. 
Then $f^*$ satisfies $\Psi_{f^k}(f^*) \leq \Psi_{f^k}(\hat{f}) $. If equality holds then $\hat{f}=F(f^k)^{\frac{1}{1-p}}\,f^k$ is a minimizer of the inner problem and the sequence terminates. In this case $f^k$ is an eigenvector, see
proof of Theorem~\ref{th:convergence_eigenvector} for Algorithm~\ref{alg:general}. Otherwise $\Psi_{f^k}(f^*) < \Psi_{f^k}(\hat{f})$
and thus $u=f^*$ fulfills the above assumption and we get $F(f^k) > F(f^*)$, as claimed.
\end{proof2}
\begin{proof2}{of Theorem~\ref{th:convergence_eigenvector} for Algorithm~\ref{alg:general}}
Note that as $F(f)\geq 0$, the sequence $F(f^k)$ is bounded from below, and by Lemma \ref{le:monotony_funct} it is monotonically decreasing and thus converges to some $\lambda^* \in [0,F(f^0)]$. Moreover, $S(f^k)=1$ for all $k$. As $S$ is continuous it attains its minimum $m$ on the unit
sphere in $\R^n$. By assumption $m>0$. We obtain
\begin{align*}
 1 = S(f^k) = S\Big(\frac{f^k}{\norm{f^k}_2}\norm{f^k}_2\Big) \geq m\, \norm{f^k}_2^p, \quad \Longrightarrow \quad \norm{f^k}_2 \leq \Big(\frac{1}{m}\Big)^\frac{1}{p}.
\end{align*}
Thus the sequence $f^k$ is bounded and there exists a convergent subsequence $f^{k_j}$. Clearly, $\lim_{j \rightarrow \infty} F(f^{k_j}) = \lim_{k \rightarrow \infty} F(f^{k})= \lambda^*$.
Let now $f^* = \lim_{j \rightarrow \infty} f^{k_j}$, and suppose that there exists $u \in \R^n$ with $\Psi_{f^*}(u) < \Psi_{f^*}(\hat{f})$ where $\hat{f}=F(f^*)^{\frac{1}{1-p}}f^*$. Then, analogously to the proof of Lemma~\ref{le:monotony_funct}, one can conclude that $F(u)<F(f^*)= \lambda^*$ which contradicts the fact that $F(f^{k_j})$ has as its limit $\lambda^*$. Thus $\hat{f}$ is a minimizer of $\Psi_{f^*}$, which implies
\begin{align*}
0 \in \partial \Psi_{f^*}\big(F(f^*)^{\frac{1}{1-p}}f^*\big) &= \partial R\big(F(f^*)^{\frac{1}{1-p}}f^*\big) - s(f^*) = \left(F(f^*)^{\frac{1}{1-p}}\right)^{p-1} \partial R(f^*) - s(f^*) \\
                                   &= \frac{1}{F(f^*)}\Big( \partial R(f^*) - F(f^*)s(f^*)\Big),
\end{align*}
so that $f^*$ is an eigenvector with eigenvalue $\lambda^*$. As this argument was independent of the subsequence, any convergent subsequence of $f^k$
converges towards an eigenvector with eigenvalue $\lambda^*$.
\end{proof2}	
%%%	
%%% END Comment for short paper 
%%%

\paragraph{Practical implementation:} By the proof of Lemma \ref{le:monotony_funct}, descent in $F$ is not only guaranteed for the optimal solution of the inner problem, but for any vector $u$ which has inner objective value $\Phi_{f^k}(u)<0=\Phi_{f^k}(f^k)$ for Alg.~\ref{alg:P1} and $\Psi_{f^k}(u)<\Psi_{f^k}(F(f^k)^{\frac{1}{1-p}}\,f^k)$ in the case of Alg.~\ref{alg:general}. This has two important practical implications. 
First, for the convergence of the IPM, it is sufficient to use a vector $u$ satisfying the above conditions instead of the optimal solution of the
inner problem. In particular, in an early stage where one is far away from the limit, it makes no sense to invest much effort to solve the inner problem accurately. 
Second, if the inner problem is solved by a descent method, a good initialization for the inner problem at step $k+1$ is given by $f^{k}$ in the case of Alg.~\ref{alg:P1} and $F(f^k)^{\frac{1}{1-p}}\,f^k$ in the case of Alg.~\ref{alg:general} as descent in $F$ is guaranteed after one step.

\section{Application 1: $1$-spectral clustering and Cheeger cuts} \label{sec:app_cheeger_cut}
Spectral clustering is a graph-based clustering method (see \cite{Lux07} for an overview) based on a relaxation of
the NP-hard problem of finding the optimal balanced cut of an undirected graph. The spectral relaxation has as its solution
the second eigenvector of the graph Laplacian and the final partition is found by optimal thresholding. While usually
spectral clustering is understood as relaxation of the so called ratio/normalized cut, it can be equally seen as relaxation of the
ratio/normalized Cheeger cut, see \cite{BH09}. Given a weighted undirected graph with vertex set $V$ and weight matrix $W$, the ratio
Cheeger cut ($\RCC$) of a partition $(C,\overline{C})$, where $C\subset V$ and $\overline{C}=V \backslash C$, is defined as 
\[
	\RCC(C,\overline{C}) :=\frac{\cut(C,\overline{C})}{\min\{\abs{C},\abs{\overline{C}}\}} \ , \hspace{1cm}  \text{where} \hspace{1cm} \cut(A,B) = \sum_{i\in A,j\in B} w_{ij},
\]
where we assume in the following that the graph is connected.
Due to limited space the normalized version is omitted, but the proposed IPM can be adapted to this case.
In \cite{BH09} we proposed $p$-spectral clustering, a generalization of spectral clustering based on the second eigenvector of the nonlinear graph $p$-Laplacian (see \cite{Amg03}; the graph Laplacian is recovered for $p=2$). The main motivation was the relation between the optimal Cheeger cut $h_\RCC=\min_{C \subset V} \RCC(C,\overline{C})$ and the Cheeger cut $h^*_\RCC$ obtained by optimal thresholding the second eigenvector of the $p$-Laplacian, see \cite{BH09,Chu97},
\begin{align*}
 \forall \, p> 1, \qquad  \frac{h_\RCC}{\max_{i \in V} d_i} \;&\leq \; \frac{h^*_{\RCC}}{\max_{i \in V} d_i} \; \leq \;p \, \bigg(\frac{h_\RCC}{\max_{i \in V} d_i}\bigg)^{\frac{1}{p}} \ ,
\end{align*}
where $d_i = \sum_{i\in V}w_{ij}$ denotes the degree of vertex $i$. While the inequality is quite loose for spectral clustering ($p=2$), it becomes tight for $p\rightarrow 1$. Indeed in \cite{BH09} much better cuts than standard spectral clustering were obtained, at the expense of higher runtime. In \cite{SB10} the idea was taken up and they considered directly
the variational characterization of the ratio Cheeger cut, see also \cite{Amg03,Chu97},
\begin{equation}\label{eq:SecondEV}
  h_{\RCC} = \min_{f\, \mathrm{ nonconstant }}  \frac{\frac{1}{2}\sum_{i,j=1}^n w_{ij}|f_i-f_j|}{\norm{f-\median(f)\ones}_1} 
  = \min_{ \stackrel{f\, \mathrm{ nonconstant }}{\mathrm{\median(f)=0}}}  \frac{\frac{1}{2}\sum_{i,j=1}^n w_{ij}|f_i-f_j|}{\norm{f}_1}   \ .
\end{equation}
In \cite{SB10} they proposed a minimization scheme based on the Split Bregman method \cite{GolOsh09}. Their method produces comparable cuts to the 
ones in \cite{BH09}, while being computationally much more efficient. However, they could not provide any convergence guarantee about their method. 

In this paper we consider the functional associated to the $1$-Laplacian $\Delta_1$,
\begin{equation}\label{eq:rayleigh_cheeger}
 F\onespect(f) =  \frac{\frac{1}{2}\sum_{i,j=1}^n w_{ij}|f_i-f_j|}{\norm{f}_1} = \frac{\inner{f, \Delta_1 f}}{\norm{f}_1},
\end{equation}
where
\[
	 (\Delta_1 f)_i = \Big\{ \sum_{j=1}^n w_{ij} u_{ij} \, |\, u_{ij}=-u_{ji}, u_{ij} \in \mathrm{sign}(f_i-f_j)\Big\}  \text{ and }  \sign(x) = \left \{ \begin{tabular}{ll} $-1$, & $x<0$,\\ $[-1,1]$, & $x=0$,\\ $1$, & $x>0$.\end{tabular} \right.
\]
and study its associated nonlinear eigenproblem $0 \in \Delta_1 f - \lambda \, \sign(f) $. 
\begin{proposition}\label{pro:MedianZero}
Any non-constant eigenvector $f^*$ of the $1$-Laplacian has median zero. Moreover, let $\lambda_2$ be the 
second eigenvalue of the $1$-Laplacian, then if $G$ is connected it holds $\lambda_2=h_{\RCC}$.
%\footnote{The median is unique if $|V|$ is odd, otherwise it is contained in $[f^*_{|V|/2-1},f^*_{|V|/2+1}]$.}. %The second eigenvalue $\lambda_2$ of the $1$-Laplacian is equal to the optimal ratio Cheeger cut $h_{\RCC}$.
\end{proposition}
%%%
%%% START Comment for short paper 
%%%
\begin{proof}
The subdifferential of the enumerator of $F\onespect$ can be computed as
\begin{align*}
\partial \big(\frac{1}{2}\sum_{i,j=1}^n w_{ij}|f_i - f_j| \big)_i = \Big\{ \sum_{j=1}^n w_{ij} u_{ij} \, |\, u_{ij}=-u_{ji}, u_{ij} \in \mathrm{sign}(f_i-f_j)\Big\},
\end{align*}
where we use the set-valued mapping  
\[ \sign(x) = \left \{ \begin{tabular}{ll} $-1$, & $x<0$,\\ $[-1,1]$, & $x=0$,\\ $1$, & $x>0$.\end{tabular} \right.\]
Moreover, the subdifferential of the denominator of $F\onespect$ is
\[ \partial \norm{f}_1 = \sign(f).\]
Note that, assuming that the graph is connected, any non-constant eigenvector $f^*$ must have $\lambda^*>0$.
Thus if $f^*$ is an eigenvector of the $1$-Laplacian, there must exist $u_{ij}$ with  $u_{ij}=-u_{ji}$ and $u_{ij} \in \mathrm{sign}(f^*_i-f^*_j)$
and $\alpha_i$ with $\alpha_i \in \sign(f^*_i)$ such that 
\[ 0 = \sum_{j=1}^n w_{ij} u_{ij} - \lambda^* \alpha_i.\]
Summing over $i$ yields due to the anti-symmetry of $u_{ij}$, $\sum_i \alpha_i = |f^*_+| - |f^*_-| + \sum_{f^*_i=0} \alpha_i = 0$, where $|f^*_+|, |f^*_-|$
are the cardinalities of the positive and negative part of $f^*$ and $|f^*_0|$ is the number of components with value zero. Thus we get
\[ \big| |f^*_+| - |f^*_-| \big| \leq |f^*_0|,\]
which implies with $|f^*_+| + |f^*_-| + |f^*_0| = |V|$ that $|f^*_+| \leq \frac{|V|}{2}$ and $|f^*_-| \leq \frac{|V|}{2}$. Thus the median of $f^*$
is zero if $|V|$ is odd. If $|V|$ is even, the median is non-unique and is contained in $[\max f^*_-,\min f^*_+]$ which contains zero.

If the graph is connected, the only eigenvector corresponding to the first eigenvalue $\lambda_1=0$ of the $1$-Laplacian is the constant one.
As all non-constant eigenvectors have median zero, it follows with Equation \ref{eq:SecondEV} that $\lambda_2 \geq h_{\RCC}$. For the other
direction, we have to use the algorithm we present in the following and some subsequent results. By Lemma \ref{lemma_piecewise_constant_rcc}
there exists a vector $f^*=\ones_C$ with $\abs{C} \leq \abs{\overline{C}}$ such that $F\onespect(f^*) = h_{\RCC}$. Obviously,
$f^*$ is non-constant and has median zero and thus can be used as initial point $f^0$ for Algorithm \ref{alg:invPowCheegerCut}. 
By Lemma \ref{lemma_monotony_funct_cheeger} starting with $f^0=f^*$ the sequence either terminates and the current iterate $f^0$ 
is an eigenvector or one finds a $f^1$ with $F\onespect(f^{1}) < F\onespect(f^0)$, where $f^{1}$ has median zero. Suppose that there exists
such a $f^1$, then $F\onespect(f^{1}) < F\onespect(f^0) = \min_{ \stackrel{f\, \mathrm{ nonconstant }}{\mathrm{\median(f)=0}}} F\onespect(f)$
which is a contradiction. Therefore the sequence has to terminate and thus by the argument in the proof of Theorem \ref{theorem_convergence_1spect}
the corresponding iterate is an eigenvector. Thus we get $h_{\RCC} \geq \lambda_2$ and thus with $\lambda_2 \geq h_{\RCC}$ we arrive 
at the desired result.
\end{proof}
%%%
%%% END Comment for short paper 
%%%

For the computation of the second eigenvector we have to modify the IPM which is discussed in the next section.

\subsection{Modification of the IPM for computing the second eigenvector of the $1$-Laplacian}
The direct minimization of \eqref{eq:rayleigh_cheeger} would be compatible with the IPM, but the 
global minimizer is the first eigenvector which is constant.
%smallest eigenvalue corresponds to the eigenvalue $\lambda_1=0$ and is the constant vector 
For computing the second eigenvector note that, unlike in the case $p=2$, 
we cannot simply project on the space orthogonal to the constant eigenvector, since mutual orthogonality of the eigenvectors does not hold in the nonlinear case. 

Algorithm \ref{alg:invPowCheegerCut} is a modification of Algorithm \ref{alg:P1} which computes a nonconstant eigenvector of the 1-Laplacian. The notation $|f^{k+1}_+|, |f^{k+1}_-|$ and $|f^{k+1}_0|$ refers to the cardinality of positive, negative and zero elements, respectively. Note that Algorithm \ref{alg:P1} requires in each step the computation of \emph{some} subgradient $s(f^k) \in \partial S(f^k)$, whereas in Algorithm \ref{alg:invPowCheegerCut} the subgradient $v^k$ has to satisfy $\inner{v^k,\ones}=0$. This condition ensures that the inner objective is invariant under addition of a constant and thus not affected by the subtraction of the median. 
Opposite to \cite{SB10} we can prove convergence to a nonconstant eigenvector of the $1$-Laplacian. However, we cannot guarantee convergence to the \emph{second} eigenvector. Thus we recommend to use multiple random initializations and use the result which achieves the best ratio Cheeger cut. 
\begin{algorithm}[htb]
   \caption{Computing a nonconstant $1$-eigenvector of the graph $1$-Laplacian}
   \label{alg:invPowCheegerCut}
\begin{algorithmic}[1]
   \STATE {\bfseries Input:} weight matrix $W$
   \STATE {\bfseries Initialization:} nonconstant $f^{0}$ with $\median(f^{0})=0$ and $\norm{f^{0}}_1=1$,  accuracy $\epsilon$
   \REPEAT
   \STATE $g^{k+1} = \argmin_{\norm{f}_2^2\leq 1} \left\{ \frac{1}{2}\sum_{i,j=1}^n w_{ij} |f_i - f_j| - \lambda^k \inner{f,v^k} \right\}$
   \STATE $f^{k+1}= g^{k+1} - \median\left(g^{k+1}\right) $
  % \STATE $f^{k+1} = h^{k+1}/\norm{h^{k+1}}_1$
    \STATE $v^{k+1}_i = \left\{\begin{tabular}{ll} $\sign(f^{k+1}_i)$, & if $f^{k+1}_i \neq 0$,\\ $-\frac{|f^{k+1}_+|-|f^{k+1}_-|}{|f^{k+1}_0|}$, & if $f^{k+1}_i=0$.\end{tabular}\right.$,
   \STATE $\lambda^{k+1}= F\onespect(f^{k+1})$  %$\lambda^{k+1}= R(f^{k+1})$
   %\UNTIL {$\abs{\min_{\norm{f}_2^2\leq 1} \left\{ \frac{1}{2}\sum_{i,j=1}^n w_{ij} |f_i - f_j| - \lambda^k \inner{f,v^k} \right\} }<\epsilon$}
   	\UNTIL $\frac{\abs{\lambda^{k+1}-\lambda^k}}{\lambda^k}< \epsilon$
\end{algorithmic}
\end{algorithm}
%%%

%%%
%%% START Comment for short paper 
%%%
\begin{lemma}\label{lemma_monotony_funct_cheeger}
The sequence $f^k$ produced by Algorithm \ref{alg:invPowCheegerCut} satisfies $F\onespect(f^k) > F\onespect(f^{k+1})$ for all $k\geq 0$ or the sequence terminates.
\end{lemma}
\begin{proof}
Note that, analogously to the proof of Lemma \ref{le:monotony_funct}, we can conclude that the inner objective is nonpositive at the optimum, where
the sequence terminates if the optimal value is zero as the previous $f^k$ is among the minimizers of the inner problem. Now observe that the objective of the inner optimization problem is invariant under addition of a constant. This follows from the fact that we always have $\inner{v^k,\ones}=0$, which can be easily verified. Hence, with $R(f)= \frac{1}{2}\sum_{i,j=1}^n w_{ij} |f_i - f_j|$, we get
\[
	R(f^{k+1}) - \lambda^k \inner{f^{k+1},v^k} = R(g^{k+1}) - \lambda^k \inner{g^{k+1},v^k} < 0 \ .
\]
Dividing both sides by $\norm{f^{k+1}}_1$  yields
\[
	\frac{R(f^{k+1})}{\norm{f^{k+1}}_1} - \lambda^k \frac{\inner{f^{k+1},v^k}}{\norm{f^{k+1}}_1} < 0 \ ,
\]
and with $\inner{f^{k+1},v^k} \leq \norm{f^{k+1}}_1 \norm{v^{k}}_\infty = \norm{f^{k+1}}_1$, the result follows.
\end{proof}
%%%
%%% END Comment for short paper 
%%%

\begin{theorem}\label{theorem_convergence_1spect}
The sequence $f^k$ produced by Algorithm \ref{alg:invPowCheegerCut} converges to an eigenvector $f^*$ of the $1$-Laplacian with eigenvalue $\lambda^* \in \left[ h_\RCC, F\onespect(f^0) \right]$. Furthermore, $F\onespect(f^k) > F\onespect(f^{k+1})$ for all $k\geq 0$ or the sequence terminates.
\end{theorem}
%%%
%%% START Comment for short paper 
%%%
\begin{proof}
Note that every constant vector $u_0$ satisfies $\Phi_{f^k}(u_0)=0$ as $\inner{v^k,\ones}=0$. The minimizer of $\Phi_{f^k}$ is either negative
or the sequence terminates in which case the previous non-constant $g^k$ is a minimizer. In any case $g^{k+1}$ cannot be constant and
in turn $f^{k+1}$ is nonconstant and has median zero. Thus for all $k$,
\[ 
F\onespect(f^k) =  \frac{\frac{1}{2}\sum_{i,j=1}^n w_{ij}|f^k_i-f^k_j|}{\norm{f^k}_1} =  \frac{\frac{1}{2}\sum_{i,j=1}^n w_{ij}|f^k_i-f^k_j|}{\norm{f^k-\median(f^k)\ones}_1} \geq h_\RCC, %\geq \lambda_2,
\] 
where we use that the median of $f^k$ is zero.
Thus $F\onespect(f^k)$ is lower-bounded by $h_\RCC$. Note that $h_\RCC \leq \lambda_2$. 
%$\lambda_2$. 
We can conclude now analogously to Theorem  \ref{th:convergence_eigenvector} that the sequence $F\onespect(f^k)$ converges to some limit 
\[
	\lambda^* = \lim_{k \rightarrow \infty} F\onespect(f^k) \geq h_\RCC  \ .
\]
As in Theorem \ref{th:convergence_eigenvector} the compactness of the set containing the sequence $g^{k}$ implies the existence of a convergent subsequence $g^{k_j}$, and using the fact that subtracting the median is continuous we have $\lim_{j \rightarrow \infty} f^{k_j} = g^* - \median(g^*)\ones =: f^*$. The proof now proceeds analogously to Theorem \ref{th:convergence_eigenvector}.
\end{proof}
%%%
%%% END Comment for short paper 
%%%

\subsection{Quality guarantee for $1$-spectral clustering}
Even though we cannot guarantee that we obtain the optimal ratio Cheeger cut, we can guarantee that $1$-spectral clustering always leads to a ratio Cheeger cut at least as good as the one found by standard spectral clustering. 
Let $(C^*_f, \overline{C^*_f})$ be the partition of $V$ obtained by optimal thresholding of $f$, where $C^*_f= \argmin\nolimits_t \,\RCC(\Ctpos, \overline{\Ctpos})$, and for $t \in \R$, $\Ctpos=\curlybrackets{ i \in V\,|\, f_i > t}$. Furthermore, $\ones_C$ denotes the vector which is $1$ on $C$ and $0$ else.
\begin{lemma}\label{lemma_piecewise_constant_rcc}
Let $C,\overline{C}$ be a partitioning of the vertex set $V$, and assume that $\abs{C} \leq \abs{\overline{C}}$. Then for any vector $f\in \R^n$ of the form
$f = \alpha \ones_C$, where $\alpha \in \R$, it holds that $F\onespect(f) = \RCC(C,\overline{C})$.
\end{lemma}
%%%
%%% START Comment for short paper 
%%%
\begin{proof}
As $F\onespect$ is scale invariant, we can without loss of generality assume that $\alpha=1$. Then we have
\begin{align*}
	F\onespect(f) & =  \frac{\frac{1}{2}\sum_{i,j=1}^n w_{ij}|f_i-f_j|}{\norm{f}_1} = \frac{\frac{1}{2}\sum_{i\in C, j \notin C} w_{ij} + \frac{1}{2}\sum_{i\notin C, j \in C }w_{ij}}{\sum_{i\in C} 1} \\
	   & = \frac{\cut(C,\overline{C})}{\abs{C}} = \frac{\cut(C,\overline{C}}{\min\curlybrackets{\abs{C},\abs{\overline{C}}}} = \RCC(C,\overline{C}) \ .
\end{align*}
\end{proof}
%%%
%%% END Comment for short paper 
%%%

\begin{lemma}\label{lemma_thresholding_decrease}
Let $f\in \R^n$ with $\median(f)=0$, and 
%$(C^*_f, \overline{C^*_f})$ denote the partition obtained by optimal thresholding of $f$ according to the $\RCC$  criterion. Furterhmore, let
$C=\argmin\{|C^*_f|, |\overline{C^*_f}|\}$. Then the vector $f^*= \ones_C$ satisfies $	F\onespect(f) \geq F\onespect(f^*)$.
\end{lemma}
%%%
%%% START Comment for short paper 
%%%
\begin{proof}
Denote by $f^+:V\rightarrow \R$ the function $f_i^+ := \max\{0,f_i\}$, and analogously, let $f^- := \max\{0,-f_i\} $. Then we have
\begin{align}
	R(f) & = \frac{1}{2}\sum_{i,j} w_{ij} \abs{f_i-f_j} = \frac{1}{2}\sum_{i,j}  w_{ij} \abs{f_i^+-f_i^- - f_j^+ + f_j^-} \label{eq:splitting} \ .
	\end{align}
Note that we always have $\abs{f_i^+-f_i^- - f_j^+ + f_j^-} = \abs{f_i^+- f_j^+} +\abs{f_i^-  - f_j^-}$, which can easily be verified by performing a case distinction over the signs of $f_i$ and $f_j$. 
Eq.\ \eqref{eq:splitting} can now be written as
\begin{align*}
	R(f) 
	  = \frac{1}{2}\sum_{i,j}  w_{ij} \abs{f_i^+-f_j^+} + \frac{1}{2}\sum_{i,j}  w_{ij} \abs{f_i^--f_j^-} =  R(f^+) + R(f^-) \ .
\end{align*}
Using the fact that $\norm{f}_1$ can be decomposed as $\norm{f^+}_1 + \norm{f^-}_1$, we obtain
\begin{equation}\label{eq:tresholding_decrease}
	\frac{R(f)}{\norm{f}_1} = \frac{R(f^+) + R(f^-)}{\norm{f^+}_1 + \norm{f^-}_1} \geq \min\left\{\frac{R(f^+)}{\norm{f^+}_1},  \frac{R(f^-)}{\norm{f^-}_1}\right\} \ .
\end{equation}
The last inequality follows from the fact that we always have for $a,b,c,d>0$,
\[
	\frac{a+b}{c+d} \geq \min \left\{\frac{a}{c}, \frac{b}{d}\right\} \ ,
\]
which can be easily shown by contradiction. Let wlog $\min \left\{\frac{a}{c}, \frac{b}{d}\right\} = \frac{a}{c}$, and assume that $\frac{a+b}{c+d} < \frac{a}{c}$. This implies $\frac{a}{c}> \frac{b}{d}$, which is a contradiction to $\frac{a}{c} \leq \frac{b}{d}$. 
Note that $\median(f)=0$, hence we have
\[
	0 \in \argmin_c\sum_{i\in V} \abs{f_i-c} \ ,
\]
which implies that $0 \in \partial \sum_{i\in V} \abs{f_i}$ and hence there exist coefficients $\abs{\alpha_i }\leq 1$ such that
\[
	0 = \sum_{f_i\neq 0} \sign(f_i) + \sum_{f_i= 0} \alpha_i \ ,
\]
which is equivalent to $\Big|\abs{\left\{i,f_i>0 \right\}} - \abs{\left\{i,f_i<0 \right\}}\Big| \leq \abs{\left\{i,f_i=0 \right\}}$. 
This inequality implies that $\abs{\left\{i,f_i>0 \right\}} \leq  \frac{\abs{V}}{2}$ and $\abs{\left\{i,f_i<0 \right\}} \leq  \frac{\abs{V}}{2}$.
We now rewrite $R(f^+)$ as follows:
\begin{align*}
	R(f^+) = \frac{1}{2} \sum_{f_i^+>f_j^+}w_{ij}\left(f^+_i-f^+_j\right)
	 = \sum_{f_i^+>f_j^+}w_{ij} \int_{f_j^+}^{f_i^+} 1 dt
	= \int_{0}^\infty \sum_{f_i^+>t \geq f_j^+}w_{ij}\;  dt \ .
\end{align*}
Note that for $t \geq 0$,
\[
	\sum_{f^+_i >t \geq f^+_j} w_{ij} =  
	 \cut(\Ctpos,\overline{\Ctpos}) =
	 \frac{\cut(\Ctpos,\overline{\Ctpos})}{\min\curlybrackets{\abs{\Ctpos}, \abs{\overline{\Ctpos}}}} 	\cdot \abs{\Ctpos} 
		\geq \RCC(C^*_f, \overline{C^*_f}) \cdot \abs{\Ctpos},
\]
where in the second step we used that $\abs{\Ctpos} \leq \abs{\left\{i,f_i>0 \right\}} \leq  \frac{\abs{V}}{2}.$  Hence we have
\begin{align*}
	R(f^+) & \geq \int_{0}^\infty  \RCC(C^*_f, \overline{C^*_f}) \cdot \abs{\Ctpos} dt = \RCC(C^*_f, \overline{C^*_f}) \int_{0}^\infty \sum_{f_i >t} 1 dt\\
	&= \RCC(C^*_f, \overline{C^*_f})\sum_{f_i >0} \int_{0}^{f_i} 1 dt = \RCC(C^*_f, \overline{C^*_f}) \norm{f^+}_1 \ .
\end{align*}
Hence it holds that $F\onespect(f^+) \geq \RCC(C^*_f, \overline{C^*_f})$, and analogously one shows that $F\onespect(f^-) \geq \RCC(C^*_{-f}, \overline{C^*_{-f}})$. 
Note that $\RCC(C^*_f, \overline{C^*_f}) = \RCC(C^*_{-f}, \overline{C^*_{-f}}) = F\onespect(f^*)$, by Lemma \ref{lemma_piecewise_constant_rcc}. Combining this with Eq.\ \eqref{eq:tresholding_decrease} yields the result.
\end{proof}
%%%
%%% END Comment for short paper 
%%%

\begin{theorem}\label{th:better_than_2spect}
Let $u$ denote the second eigenvector of the standard graph Laplacian, and f denote the result of Algorithm~\ref{alg:invPowCheegerCut} after initializing with the vector $\frac{1}{\abs{C}}\ones_{C}$, where $C=\argmin\{|C^*_u|, |\overline{C^*_u}|\}$. Then $\RCC(C^*_u,\overline{C^*_u}) \geq  \RCC(C^*_f,\overline{C^*_f})$.
\end{theorem}
%%%
%%% START Comment for short paper 
%%%
\begin{proof}
Using Lemma \ref{lemma_monotony_funct_cheeger} and \ref{lemma_piecewise_constant_rcc} , we have the following chain of inequalities: 
 \[
 	\RCC(C^*_u,\overline{C^*_u}) \stackrel{\ref{lemma_piecewise_constant_rcc}}{=} F\onespect\left(\frac{1}{\abs{C}}\ones_{C}\right)
 =  F\onespect(\ones_C) \stackrel{\ref{lemma_monotony_funct_cheeger}}{\geq} F\onespect(f) .
 \]
 With $C_1 := \argmin\{|C^*_f|, |\overline{C^*_f}|\}$, we obtain by Lemma~\ref{lemma_thresholding_decrease} and \ref{lemma_piecewise_constant_rcc}:
 \[
   F\onespect(f) \stackrel{\ref{lemma_thresholding_decrease}}{\geq} F\onespect\left(\ones_{C_1}\right)  \stackrel{\ref{lemma_piecewise_constant_rcc}}{=} \RCC(C^*_f,\overline{C^*_f}) \ .
 \]
\end{proof}
%%%
%%% END Comment for short paper 
%%%

\subsection{Solution of the inner problem}
The inner problem is convex, thus a solution can be computed by any standard method for solving convex nonsmooth programs, e.g.\ subgradient methods \cite{Ber99}. However, 
in this particular case we can exploit the structure of the problem and use the equivalent dual formulation of the inner problem. 
\begin{lemma}\label{le:DualInnerProblem}
Let $E \subset V\times V$ denote the set of edges and $A:\R^E \rightarrow \R^V$ be defined as $(A \alpha)_i = \sum_{j\,|\, (i,j) \in E} w_{ij}\alpha_{ij}$. The inner problem is equivalent to 
\[ \min_{\{ \alpha \in \R^E \,|\, \norm{\alpha}_\infty \leq 1, \; \alpha_{ij}=-\alpha_{ji} \}} \Psi(\alpha):=\norm{A \alpha - F(f^k)v^k}^2_2. \]
The Lipschitz constant of the gradient of $\Psi$ is upper bounded by $2 \max_r \sum_{s=1}^n w_{rs}^2$.
\end{lemma}
%%%
%%% START Comment for short paper 
%%%
\begin{proof}
First, we note that 
\begin{align*}
  \frac{1}{2}\sum_{i,j=1}^n w_{ij} |u_i - u_j| &= \max_{\{\beta \in \R^E \, |\, \norm{\beta}_\infty \leq 1 \}} \frac{1}{2}\sum_{(i,j) \in E} w_{ij} (u_i - u_j) \beta_{ij} \ .
\end{align*}
Introducing the new variable $\alpha_{ij} = \frac{1}{2}(\beta_{ij}-\beta_{ji})$, this can be rewritten as
\[
	\max_{\{ \alpha \in \R^E \,|\, \norm{\alpha}_\infty \leq 1,\; \alpha_{ij}=-\alpha_{ji} \}} \sum_{(i,j) \in E} w_{ij} \alpha_{ij} u_i 
	= \max_{\{ \alpha \in \R^E \,|\, \norm{\alpha}_\infty \leq 1,\; \alpha_{ij}=-\alpha_{ji} \}} \inner{u,A\alpha} \ ,
\]
where we have introduced the notation $(A \alpha)_i = \sum_{j\,|\, (i,j) \in E} w_{ij}\alpha_{ij}$. Both $u$ and $\alpha$ are constrained to lie in non-empty compact, convex sets, and thus we can reformulate the inner objective by the standard min-max-theorem (see e.g. Corollary 37.3.2. in \cite{Roc70}) as follows:
\begin{align*}
 	& \min_{\norm{u}_2 \leq 1}\;\max_{\{ \alpha \in \R^E \,|\, \norm{\alpha}_\infty \leq 1, \, \alpha_{ij}=-\alpha_{ji} \}}  \inner{u,A\alpha}-F(f^k)\inner{u,v^k}\\
	&= \max_{\{ \alpha \in \R^E \,|\, \norm{\alpha}_\infty \leq 1, \; \alpha_{ij}=-\alpha_{ji} \}}\;\min_{\norm{u}_2 \leq 1} \inner{u,A\alpha-F(f^k)v^k}\\
	&= \max_{\{ \alpha \in \R^E \,|\, \norm{\alpha}_\infty \leq 1, \; \alpha_{ij}=-\alpha_{ji} \}} \, -\norm{A\alpha - F(f^k)v^k}_2 \ .
\end{align*}
In the last step we have used that the solution of the minimization of the linear function over the Euclidean unit ball is given by 
\[ u^* = -\frac{A \alpha - F(f^k)v^k}{\norm{A \alpha - F(f^k)v^k}}_2 ,\] 
if $\norm{A \alpha - F(f^k)v^k}\neq 0$ and otherwise $u^*$ is an arbitrary element of the Euclidean unit ball. Transforming the maximization problem into a minimization problem finishes the proof of the first statement. Regarding the Lipschitz constant, a straightforward computation shows that 
\[ (\nabla \Psi(\alpha))_{rs} = 2 w_{rs}\Big(\sum_{j\,|\, (r,j) \in E} w_{rj}\alpha_{rj} - F(f^k)v^k_r\Big).\]
Thus,
\begin{align*}
\norm{\nabla \Psi(\alpha) - \nabla \Psi(\beta)}^2 \; &= \; 4 \sum_{(r,s) \in E} w_{rs}^2\,\Big( \sum_{j\,|\, (r,j) \in E} w_{rj}(\alpha_{rj}-\beta_{rj}) \Big)^2\\
                                                    &\leq \; 4 \sum_{(r,s) \in E} w_{rs}^2\, \Big( \sum_{j\,|\, (r,j) \in E} w^2_{rj} \sum_{i \,|\, (r,i) \in E
                                                    } (\alpha_{ri}-\beta_{ri})^2 \Big) \\
                                                    &=\; 4 \sum_{r=1}^n \Big(\sum_{s\,|\, (r,s) \in E} w_{rs}^2\Big)^2 \sum_{i \,|\, (r,i) \in E} (\alpha_{ri}-\beta_{ri})^2 \\
                                                    & \leq \; 4 \Big(\max_r \sum_{s=1}^n w_{rs}^2\Big)^2 \sum_{(r,i) \in E} (\alpha_{ri}-\beta_{ri})^2.
\end{align*}
\end{proof}
%%%
%%% END Comment for short paper 
%%%

Compared to the primal problem, the objective of the dual problem is smooth.
Moreover, it can be efficiently solved using FISTA (\cite{BT09}), a two-step subgradient method with guaranteed convergence rate $O(\frac{1}{k^2})$ where $k$ is the number of steps. The only input of FISTA is an upper bound on the Lipschitz constant of the gradient of the objective.
FISTA provides a good solution in a few steps which guarantees descent in functional \eqref{eq:SecondEV} and thus makes the 
modified IPM very fast.
%%%
%%% Comment in SHORT version
%%%
The resulting Algorithm is shown in Alg. \ref{alg:FISTA}.
\begin{algorithm}[tb]
   \caption{Solution of the dual inner problem with FISTA}
   \label{alg:FISTA}
\begin{algorithmic}[1]
   \STATE {\bfseries Input:} Lipschitz-constant $L$ of $\nabla \Psi$,
   \STATE {\bfseries Initialization:} $t^1=1$, $\alpha^1 \in \R^E$,
   \REPEAT
   \STATE  \begin{align*}
          \beta^{t+1}_{rs} &= \alpha^t_{rs} - \frac{1}{L} \nabla \Psi(\alpha^t)_{rs} \\
                           &=\alpha^t_{rs} - \frac{2}{L}w_{rs}\big( \sum_{j\,|\, (r,j) \in E} w_{rj} \alpha^t_{rj} - F(f^k)v^k_r\big)
          \end{align*}
   \STATE $t_{k+1} = \frac{1+ \sqrt{1+4 t_k^2}}{2}$,
   \STATE $\alpha^{t+1}_{rs} = \beta^{t+1}_{rs} + \frac{t_k-1}{t_{k+1}}\Big(\beta^{t+1}_{rs} - \beta^t_{rs}\Big)$.
   \UNTIL {stop if gap between original and dual problem is smaller than $\epsilon$}
\end{algorithmic}
\end{algorithm}
%%%
%%% END Comment for short paper 
%%%

\section{Application 2: Sparse PCA} \label{sec:app_sparse_pca}
Principal Component Analysis (PCA) is a standard technique for dimensionality reduction and data analysis \cite{Jol02}. PCA finds the $k$-dimensional subspace of maximal variance in the data. For $k=1$, given a data matrix $X\in \R^{n\times p}$ where each column has mean $0$, in PCA one computes
\begin{equation}\label{eq:pca}	
	f^*= %\argmax_f \frac{\var(Xf)}{\norm{x}_2^2} = 
	     \argmax_{f \in \R^p} \frac{\inner{f, X^T X f}}{\norm{f}_2^2 } \ ,
\end{equation}
where the maximizer $f^*$ is the largest eigenvector of the covariance matrix $\Sigma = X^T X\in \R^{p \times p}$. 
The interpretation of the PCA component $f^*$ is difficult as usually all components are nonzero. 
In sparse PCA one wants to get a small number of features which still capture most of the variance.
For instance, in the case of gene expression data one would like the principal components to consist only of a few significant genes, making it easy to interpret by a human. 
Thus one needs to enforce sparsity of the PCA component, which yields a trade-off between explained variance and sparsity. 

While standard PCA leads to an eigenproblem, adding a constraint on the cardinality, i.e.\ the number of nonzero coefficients, makes the problem NP-hard. The first approaches performed simple thresholding of the principal components which was shown to be misleading \cite{CadJol95}. Since then several methods have been proposed, mainly based on penalizing the $L_1$ norm of the principal components, including SCoTLASS \cite{JolTreUdd03} and SPCA \cite{ZouHasTib06}. D'Aspremont et al.\cite{Asp08} focused on the $L_0$-constrained formulation and proposed a greedy algorithm to compute a full set of good candidate solutions up to a specified target sparsity, and derived sufficient conditions for a vector to be globally optimal. Moghaddam et al. \cite{MogWeiAvi06} used branch and bound to compute optimal solutions for small problem instances. Other approaches include D.C.  \cite{SriTorLan07} and EM-based methods \cite{SigBuh08}. Recently, Journee et al. \cite{Jou10} proposed two single unit (computation of one component only) and two block (simultaneous computation
of multiple components) methods based on $L_0$-penalization and $L_1$-penalization. 

Problem \eqref{eq:pca} is equivalent to
\[
	f^* = \argmin_{f \in \R^p} \frac{\norm{f}_2^2 } {\inner{f,\Sigma f}} = \argmin_{f \in \R^p} \frac{\norm{f}_2 } {\norm{Xf}_2} \ .
\]
In order to enforce sparsity we use instead of the $L_2$-norm a convex combination of an $L_1$ norm and $L_2$ norm in the enumerator, which yields the functional
\begin{equation}\label{eq:sparse_pca}   
F(f) = \frac{(1-\alpha)\norm{f}_2 + \alpha \norm{f}_1}{\norm{X f}_2}\ ,
\end{equation}
with sparsity controlling parameter $\alpha \in [0,1]$. Standard PCA is recovered for $\alpha=0$, whereas $\alpha=1$ yields the sparsest non-trivial solution: the component with the maximal variance. One easily sees that the formulation \eqref{eq:sparse_pca} fits in our general framework, as both enumerator and denominator are $1$-homogeneous functions. The inner problem of the IPM becomes
\begin{equation}\label{eq:innerSPCA}
g^{k+1} = \argmin_{\norm{f}_2 \leq 1} \,(1-\alpha)\norm{f}_2 + \alpha \norm{f}_1 - \lambda^k \inner{f,\mu^k}\ , \quad  \text{where } \quad \mu^k = \frac{\Sigma f^k}{\sqrt{\inner{f^k,\Sigma f^k}}}. 
\end{equation}
This problem has a closed form solution. In the following we use the notation $x_+ = \max\{0,x\}$.
\begin{lemma}\label{lemma:spca_closedform}
The convex optimization problem \eqref{eq:innerSPCA} has the analytical solution
\[
	g_i^{k+1} = \frac{1}{s} \sign(\mu^k_i) \big(\lambda^k \abs{\mu^k_i}-\alpha\big)_+, \quad \text{where } \quad s = \sqrt{\sum\nolimits_{i=1}^n (\lambda^k|\mu^k_i|-\alpha)_+^2} \ .
\]
\end{lemma}
%%%
%%% START Comment for short paper 
%%%
\begin{proof}
We note that the objective is positively 1-homogeneous and that the optimum is either zero by plugging in the previous iterate or negative
in which case the optimum is attained at the boundary. Thus wlog we can assume that at the optimum $\norm{f}_2=1$. Thus the problem reduces
to 
\[ \min_{\norm{f}_2 \leq 1} \,\alpha \norm{f}_1 - \lambda^k \inner{f,\mu^k}.\]
First, we derive an equivalent ``dual'' problem, noting 
\begin{align*}
\alpha \norm{f}_1 - \lambda^k \inner{\mu^k,f} = \max_{\norm{v}_\infty \leq 1} \inner{f,\alpha v - \lambda^k \mu^k} \ .
\end{align*}
Using the fact that the objective is convex in $f$ and concave in $v$ and the feasible set is compact, we obtain by the min-max equality:
\begin{align*}
\min_{\norm{f}_2\leq 1} \;\max_{\norm{v}_\infty \leq 1} \inner{f,\alpha v - \lambda^k \mu^k} 
& = \max_{\norm{v}_\infty \leq 1} \;\ \min_{\norm{f}_2\leq 1} \inner{f, \alpha v - \lambda^k \mu^k} \\
& = \max_{\norm{v}_\infty \leq 1} - \norm{\alpha v - \lambda^k \mu^k}_2 \ .
\end{align*}
The objective of the dual problem is separable in $v$ and the constraints of $v$ as well. Thus each component can be optimized
separately which gives
\[ v_i = \sign(\mu^k_i)\min\left\{1,\frac{\lambda^k |\mu^k_i|}{\alpha}\right\}.\]
Using that $f^*= (-\alpha v + \lambda^k \mu^k) / \norm{\lambda^k \mu^k-\alpha v}_2$, we get the solution
\[ f_i = \frac{\sign(\mu^k_i)(\lambda^k |\mu^k_i|-\alpha)_+}{\sqrt{\sum_{i=1}^n (\lambda^k |\mu^k_i|-\alpha)_+^2}}.\]
\end{proof}
%%%
%%% END Comment for short paper 
%%%

As $s$ is just a scaling factor, we can omit it 
and obtain the simple and efficient scheme to compute sparse principal components shown in Algorithm~\ref{alg:sparse_pca}.
While the derivation 
is quite different from \cite{Jou10}, the resulting algorithms are very similar. The
subtle difference is that in our formulation the thresholding parameter of the inner problem depends on the current eigenvalue estimate whereas it is fixed in \cite{Jou10}.
Empirically, this leads to the fact that we need slightly less iterations to converge.
\begin{algorithm}[htb]
   \caption{Sparse PCA}
   \label{alg:sparse_pca}
\begin{algorithmic}[1]
   \STATE {\bfseries Input:} data matrix $X$, sparsity controlling parameter $\alpha$, accuracy $\epsilon$
   \STATE {\bfseries Initialization:} $f^0 = \text{random}$ with $S(f^k) = 1$, $\lambda^0 =F(f^k)$
   \REPEAT
   \STATE $g_i^{k+1} = \sign(\mu^k_i) \big(\lambda^k \abs{\mu^k_i}-\alpha\big)_+$, %\hspace{0.5cm} where $s = \sqrt{\sum_{i=1}^n (\lambda^k |\mu^k_i|-\alpha)_+^2}$.
   \STATE $f^{k+1} = \frac{g^{k+1}}{\norm{Xg^{k+1}}_2}$
   \STATE $\lambda^{k+1}= (1-\alpha)\norm{f^{k+1}}_2 + \alpha \norm{f^{k+1}}_1$
   \STATE $\mu^{k+1} = \frac{\Sigma f^{k+1}}{\norm{Xf^{k+1}}_2}$
 	\UNTIL $\frac{\abs{\lambda^{k+1}-\lambda^k}}{\lambda^k}< \epsilon$
\end{algorithmic}
\end{algorithm}

\section{Experiments}\label{sec:experiments}

\paragraph{1-Spectral Clustering:}
We compare our IPM with the total variation (TV) based algorithm by \cite{SB10}, $p$-spectral clustering with $p=1.1$ \cite{BH09} as well as standard spectral clustering with optimal thresholding the second eigenvector of the graph Laplacian ($p=2$). The graph and the two-moons dataset is 
constructed as in \cite{BH09}. 
The following table shows the average ratio Cheeger cut (RCC) and error (classification as in \cite{BH09}) for 100 draws of a two-moons dataset with 2000 points. In the case of the IPM, we use the best result of 10 runs with random initializations and one run initialized with the second eigenvector of the unnormalized graph Laplacian. For \cite{SB10} we initialize once with the second eigenvector of the normalized graph Laplacian as proposed in \cite{SB10} and 
10 times randomly. 
IPM and the TV-based method yield similar results, slightly better than $1.1$-spectral and clearly outperforming standard spectral clustering. In terms of runtime, IPM and \cite{SB10} are on the same level.
{\footnotesize
\begin{center}
\begin{tabular}{c |c |c | c | c}  
				& Inverse Power Method &   Szlam \& Bresson \cite{SB10} & $1.1$-spectral \cite{BH09} & Standard spectral \\
%				& \multicolumn{3}{c|}{Inverse Power Method} &   \multicolumn{3}{c|}{Szlam \& Bresson \cite{SB10}} & $p$-Spectral  & Second \\
				\hline
Avg. $\RCC$	& 0.0195 ($\pm$ 0.0015)	& 0.0195 	($\pm$ 0.0015) & 0.0196   ($\pm$ 0.0016)	& 0.0247  ($\pm$ 0.0016)\\
Avg. error	& 0.0462 ($\pm$ 0.0161) & 0.0491 	($\pm$ 0.0181) & 0.0578 	($\pm$ 0.0285)  & 0.1685  ($\pm$ 0.0200)
\end{tabular} \end{center}
}

\begin{figure}[ht]
	\centering
		\includegraphics[width=0.27\textwidth]{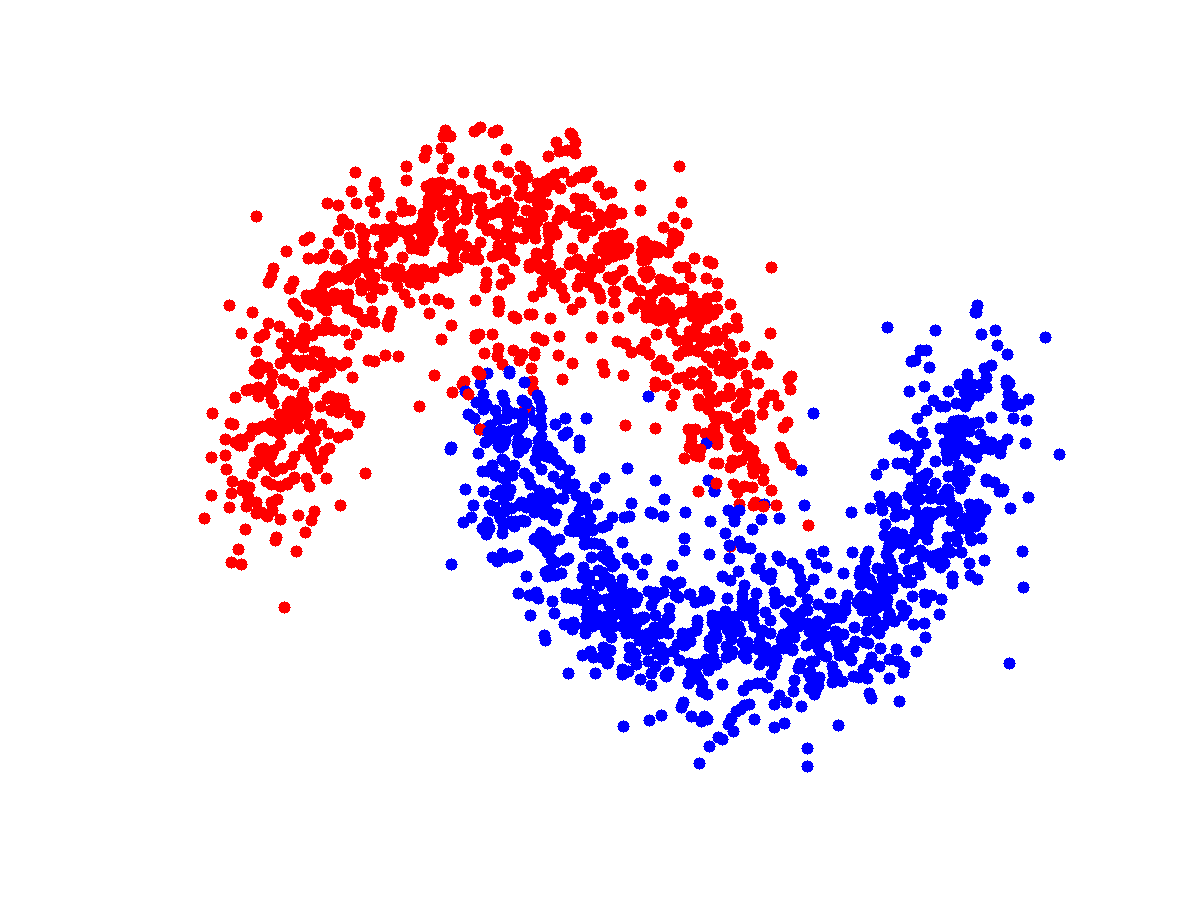}
		\includegraphics[width=0.27\textwidth]{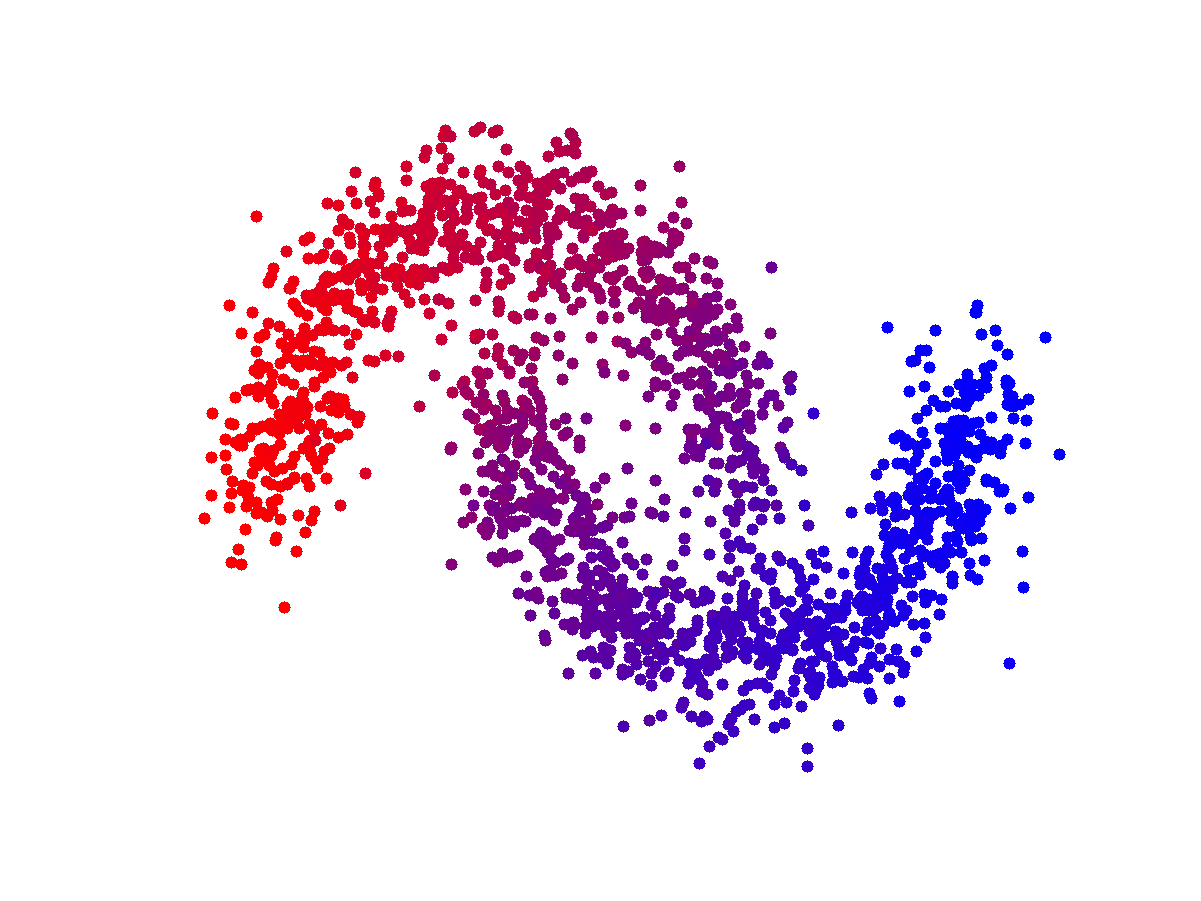}
		\hspace{1cm}
		\includegraphics[width=0.28\textwidth]{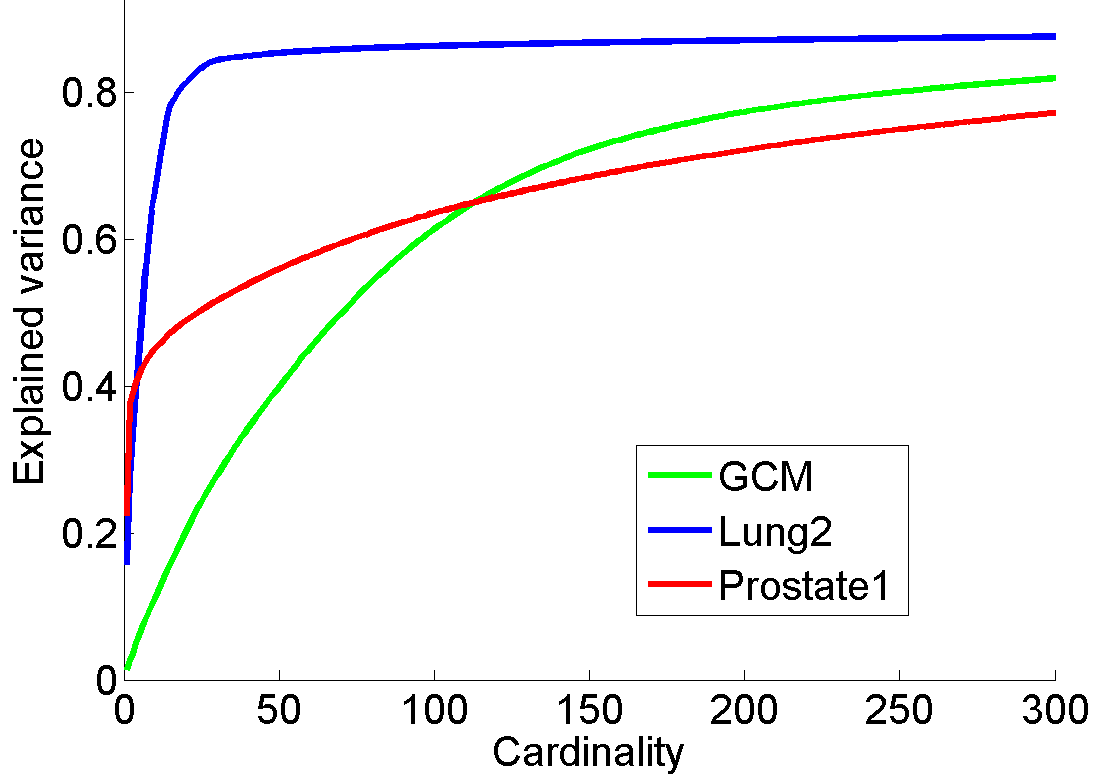}
	\caption{Left and middle: Second eigenvector of the 1-Laplacian and 2-Laplacian, respectively. Right: Relative Variance (relative to maximal possible variance) versus number of non-zero components for the three datasets Lung2, GCM and Prostate1.}
		\label{fig:results-two-moons}
\end{figure}
%\vspace{-6mm}
%\vspace{-4mm}

Next we perform unnormalized $1$-spectral clustering on the full USPS and MNIST-datasets ($9298$ resp.\ $70000$ points). As clustering criterion we use the multicut version of $\Rcut$, given as
\[
	\Rcut(C_1,\dots,C_K)= \sum_{i=1}^K \frac{\cut(C_i,\overline{C_i})}{\abs{C_i}} \ .
\]
We successively subdivide clusters until the desired number of clusters ($K=10$) is reached. In each substep the eigenvector obtained on the subgraph is thresholded such that the multi-cut criterion
is minimized.
This recursive partitioning scheme is used for all methods. As in the previous experiment, we perform one run initialized with the thresholded second eigenvector of the unnormalized graph Laplacian in the case of the IPM and with the second eigenvector of the normalized graph Laplacian in the case of \cite{SB10}. In both cases we add 100 runs with random initializations. The next table shows the obtained $\Rcut$ and errors.

{\footnotesize
\begin{center}
\begin{tabular}{c |c |c |c |c |c }  

				& & Inverse Power Method &   S.\&B. \cite{SB10} & $1.1$-spectral \cite{BH09} & Standard spectral\\
 				\hline
MNIST	& Rcut	& 0.1507	& 0.1545	& 0.1529	& 0.2252\\
			& Error	& 0.1244	& 0.1318	& 0.1293	& 0.1883\\
\hline
USPS 	& Rcut	& 0.6661	& 0.6663	& 0.6676	& 0.8180\\
			& Error	& 0.1349	& 0.1309	& 0.1308	& 0.1686
\end{tabular}
\end{center}
}
Again the three nonlinear eigenvector methods clearly outperform standard spectral clustering. Note that our 
method requires additional effort (100 runs) but we get better results. For both datasets our method achieves the best $\Rcut$. However, if one wants to do only a single run, by Theorem~\ref{th:better_than_2spect} for bi-partitions one achieves a cut at least as good as the one of standard spectral clustering if one initializes with the thresholded 2nd eigenvector of the 2-Laplacian.

\paragraph{Sparse PCA:} We evaluate our IPM for sparse PCA on gene expression datasets obtained from \cite{WeiLi}. We compare with two recent algorithms: the $L_1$ based single-unit power algorithm of \cite{Jou10} as well as the EM-based algorithm in \cite{SigBuh08}. 
For all considered datasets, the three methods achieve very similar performance in terms of the tradeoff between explained variance and sparsity of the solution, see Fig.\ref{fig:results-two-moons} (Right). In fact the results are so similar that for each dataset, the plots of all three methods coincide in one line.  
In \cite{Jou10} it also has been observed that the best state-of-the-art algorithms produce the same trade-off curve if one uses the same initialization strategy.

\paragraph{Acknowledgments:}
This work has been supported by the Excellence Cluster on
Multimodal Computing and Interaction at Saarland University.

 {\small
\bibliography{literatur}
}
\end{document}